\documentclass[12pt]{amsart}
\usepackage{amssymb,amsfonts,euscript,mathrsfs,latexsym,amscd,amsmath}
\usepackage{epsf, color, graphicx, amsthm,amsopn,bbm, url,tabularx}
\usepackage[spacing,kerning]{microtype}
\usepackage[dvipsnames]{xcolor}
\usepackage[utf8]{inputenc}
\usepackage[arrow,matrix]{xy}
\usepackage{amssymb}
\usepackage{listings}
\usepackage{hyperref}
\definecolor{darkblue}{RGB}{0,0,160}
\hypersetup{
colorlinks,%
citecolor=black,%
filecolor=black,%
linkcolor=darkblue,%
urlcolor=darkblue
}

\usepackage[
text={440pt,575pt},
headheight=9pt,
centering
]{geometry}

\newcommand{\RR}{\mathbb{R}}

\newcommand{\PP}{\mathbb{P}}

\newcommand\macaulay{\texttt{Macaulay2 }}

\DeclareMathOperator{\conv}{conv}

\DeclareMathOperator{\diag}{diag} 

\newtheorem{thm}{Theorem}[section]

\newtheorem{lem}[thm]{Lemma}

\newtheorem{prop}[thm]{Proposition}

\theoremstyle{definition}

\newtheorem{exmp}[thm]{Example}

\newtheorem{rem}[thm]{Remark}

\usepackage[]{algorithm2e}
\usepackage{subcaption}
\definecolor{codegreen}{rgb}{0,0.6,0}
\definecolor{codepurple}{rgb}{0.58,0,0.82}
\definecolor{codered}{RGB}{185,19,6}

 \lstdefinelanguage{myLang}{
   basicstyle=\footnotesize\ttfamily,
   xleftmargin=2em,
   xrightmargin=2em,
   columns=fullflexible,
   keepspaces=true, 
   classoffset=2,
   morekeywords={minors, genericMatrix, matrix, ideal, eliminate, map, degrees, random, numgens, minors, binomial, entries, transpose, submatrix, det},
   keywordstyle={\color{blue}\bfseries},
   classoffset=3,
   morekeywords={for,from, to, do, while},
   keywordstyle={\color{codepurple}\bfseries},
   classoffset=4,
   morekeywords={Height},
   keywordstyle={\color{Emerald}\bfseries},
   classoffset=5,
   morekeywords={QQ, ZZ},
   keywordstyle={\color{codegreen}\bfseries},
   sensitive=false, % keywords are not case-sensitive
   morecomment=[l]{--}, % l is for line comment
   commentstyle=\color{codered},
   stepnumber=1,
   numbers=left,
   captionpos=b,
   showspaces=false,
   showstringspaces=false,
   morestring=[b]",
   frame=single
}

 \usepackage[arrow,matrix]{xy}
 \usepackage[c]{esvect}

\begin{document}

\title{Pictures of Combinatorial Cubes}
\date{July 2017}
\author{Andr\'e Wagner}
\address{Technische Universität Berlin\\ Berlin, Germany}
\urladdr{\url{http://page.math.tu-berlin.de/~wagner}}
\label{Chp:critical} % For referencing the chapter elsewhere, use \ref{Chapter1} 

\begin{abstract}
We prove that the 8-point algorithm always fails to reconstruct a unique fundamental matrix $F$
independent on the camera positions, when its input are image point configurations that are 
perspective projections of  the vertices of a combinatorial cube in $\RR^3$. We give an algorithm that improves  the
7- and 8-point algorithm in such a pathological situation. Additionally we analyze the regions of focal point positions
where a reconstruction of $F$  is possible at all,  when the world points are the vertices of a combinatorial
cube in $\RR^3$.
\end{abstract}
\maketitle
%%%%%%%%%%%%
\section{Introduction}
%%%%%%%%%%%%

The \emph{8-point algorithm} \cite[Algorithm 11.1]{hartley2003multiple} is one of the key
algorithms in epipolar geometry. It is successfully used in a vast number of applications
to compute the \emph{fundamental matrix} $F\in\RR^{3\times 3}$ between two views.  The 8-point algorithm purely relies
on methods from linear algebra and  is extremely fast. At the same time it only gives
slightly inferior results compared to more involved algorithms that rely on nonlinear
optimization.

An algorithm to reconstruct the fundamental matrix is \emph{defeated by a world point
configuration} in $(\PP^3)^n$, if this algorithm fails to produce a unique fundamental
matrix from the projections in $(\PP^2)^n$ of that world point configuration independent
on the choice of cameras.  The 8-point algorithm is \emph{defeated} by the eight vertices
of a unit cube. We extend this result to arbitrary convex 3-polytopes bounded by six
quadrilateral faces, whose vertex-facet incidence is the same as that of a cube. Such a polytope in $\RR^3$
is called   \emph{combinatorial cube}. The
vertex-facet incidence of a polytope is the undirected bipartite graph formed by the
containment of vertices within the facets \cite[\S 3.5]{joswig2013polyhedral}.

\begin{figure} \centering
\begin{subfigure}[b]{.4\textwidth} \centering
  \includegraphics[width=.7\linewidth]{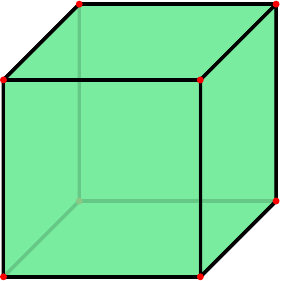}
  \caption{a standard cube}
  \label{fig:Cube1}
\end{subfigure}
\begin{subfigure}[b]{.4\textwidth} \centering
  \includegraphics[width=.7\linewidth]{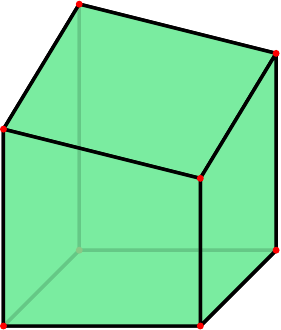}
  \caption{a combinatorial cube}
  \label{fig:Cube2}
\end{subfigure}
\caption{Two projective non equivalent combinatorial cubes}
\label{fig:CubesAndGraph}
\end{figure}

In the last section of this paper we suggest how to handle the situation when the 8-point
algorithm is defeated by the vertices of a combinatorial cube by using a modified version of the
7-point algorithm in this case.

There are multiple papers in multiview geometry which are concerned with critical
configurations \cite{maybank2012theory,hartley2007critical}.  Critical configurations in
two-view geometry are point configurations consisting of $n$ world points together with
the two focal points $f_i\in\PP^3$ of the cameras. For these configurations there exist
two ambiguous fundamental matrices. They are a feature of the geometry itself and not of
the choice of algorithm used to reconstruct the fundamental matrix $F\in\RR^{3 \times
3}$. Hartley and Kahl \cite{hartley2007critical} give the complete description of critical
configurations in multiview geometry. In the two-view case a necessary condition for a
focal-world point configuration to be critical, is that the $n$ world points $P_i\in
\PP^3$ and the two camera centers $f_i$ lie on a \emph{ruled quadric}
\cite{krames1941ermittlung}. Thus it is not possible to reconstruct a unique fundamental
matrix, independent of the chosen algorithm. Here we consider world
point configurations where the 8-point algorithm always fails to reconstruct the
fundamental matrix independent of the camera centers. These configurations are related to
critical configurations but they do not align. There are only a few results about
these configurations and usually rely on dimensional degeneracies, e.g. too many points on
a plane or on a line. In this case the quadric running through the $n$ points and the two
cameras centers is degenerate. For example in \cite{philip1998critical} it is shown, that
all points but one on a plane defeat the 8-point algorithm. If we add the two focal
points to this configuration then the point off the plane and the focal points span a
plane. Hence we can fit a ruled quadric of two intersecting planes through the $n+2$
points.

By $P_i\in\PP^k$ we denote the $i$-th point in a point configuration $P=(P_1,\ldots,P_n)$ of $n$
points in $\PP^k$.

Pinhole cameras are represented as  $3\times 4$ matrices $A_j$ with real entries. If two
image points $X_i, Y_i$ in two different pictures are perspective projections of the same
world point $P_i$ they  must satisfy the perspective relation $A_1P_i=\lambda_1X_i, \,
A_2P_i=\lambda_2Y_i$ with $\lambda_j\in \RR\setminus\{0\}$. As these equations must be
satisfied at the same time one can deduce a bilinear relation which both points must satisfy. 
\[
Y_i^TFX_i=0,
\]
where $F$ is the \emph{fundamental matrix}. To reconstruct the fundamental matrix from the
image point configurations $X,Y\in (\PP^2)^n$ different algorithms are available \cite[\S
11]{hartley2003multiple}. The most commonly used algorithm to reconstruct the fundamental
matrix is the 8-point algorithm \cite[Algorithm 11.1]{hartley2003multiple}.  However
standard implementations of the 8-point algorithm assume that the image point
configurations are suitably generic, such that the fundamental matrix $F$ can be
determined as the solution of a system of linear equations. For exact data the 8-point
algorithm then is as follows.  Via vectorizing $F\in\RR^{3\times 3}$ denoted by $\vv
F\in\RR^9$ it can be computed as the kernel of the matrix
\[
 Z=\begin{bmatrix}X_1\otimes Y_1\\ \vdots\\ X_n\otimes Y_n\end{bmatrix}\in \RR^{n\times 9},
\]
where $Z$ is the row wise tensor product of the image point configurations $X,Y\in(\PP^2)^n$. The
kernel of $Z$ is one-dimensional and $\vv F=\ker(Z)$. This yields the fundamental matrix $F\in
\RR^{3 \times 3}$. One necessary condition for the 8-point algorithm to successfully compute a
fundamental matrix is
\begin{equation} \label{dimKer} \dim(\ker(Z))=1.
\end{equation} 
Algorithm \ref{algo:8-point} states the 8-point algorithm in the absence of noise.

\begin{figure}
\begin{algorithm}[H]\label{algo:8-point}
 \KwIn{Two image point configurations $X,Y\in(\PP^2)^n$;}
 \KwOut{The fundamental matrix $F$;}
\Begin{
\begin{enumerate}
\item Compute $Z\in \RR^{8\times 9} $ from $X,Y$;
\item Compute the kernel $\vv F$ of $Z$;
\item The nine coordinates of $\vv F$ form the fundamental matrix $F$;
\end{enumerate}
}
 \caption{8-point algorithm (without noise)}
\end{algorithm}
\end{figure}

However, even if condition \ref{dimKer} is violated and the 8-point algorithm
fails to construct a fundamental matrix, it is sometimes possible to retrieve a unique
fundamental matrix using a different algorithm by additionally enforcing the rank two
constraint of the fundamental matrix $F$, namely if there exists
a unique rank two matrix in the kernel of $Z$.

\begin{exmp}
Let $P=\conv(\pm e_1\pm e_2 \pm e_3)$ be a standard cube and 
\[ \begin{array}{lr}
A_1=\begin{bmatrix}
     1 &    0  &   0  &   2\\
     0  &   1  &   0  &   3\\
     0  &   0  &   1  &   2
\end{bmatrix},& A_2=\begin{bmatrix}
     1 &    0  &   0  &   2\\
     0  &   1  &   0  &   3\\
     0  &   0  &   1  &   1
\end{bmatrix}
\end{array}\]
then 
\[
\begin{array}{rl}
X=\begin{bmatrix}
1  &   3   &  1  &   3  &   1  &   3   &  1  &   3\\
2  &   2   &  4  &   4  &   2  &   2   &  4  &   4\\
1  &   1   &  1  &   1  &   3  &   3   &  3  &   3  
\end{bmatrix}, &
Y=\begin{bmatrix}  
1  &   3   &  1  &   3  &   1  &   3   &  1  &   3\\
2  &   2   &  4  &   4  &   2  &   2   &  4  &   4\\
0  &   0   &  0  &   0  &   2  &   2   &  2  &   2
\end{bmatrix}\end{array}
\]
and
\[
Z=\footnotesize\arraycolsep=7pt\def\arraystretch{0.7}\begin{bmatrix}
1  &   2   &  1  &   2  &   4  &   2   &  0  &   0   &  0\\ 
9  &   6   &  3  &   6  &   4  &   2   &  0  &   0   &  0\\ 
1  &   4   &  1  &   4  &  16  &   4   &  0  &   0   &  0\\ 
9  &  12   &  3  &  12  &  16  &   4   &  0  &   0   &  0\\ 
1  &   2   &  3  &   2  &   4  &   6   &  2  &   4   &  6\\ 
9  &   6   &  9  &   6  &   4  &   6   &  6  &   4   &  6\\ 
1  &   4   &  3  &   4  &  16  &  12   &  2  &   8   &  6\\ 
9  &  12   &  9  &  12  &  16  &  12   &  6  &   8   &  6
\end{bmatrix}
\]
The kernel of $Z$ is two-dimensional and the 8-point algorithm is defeated by the vertices
of  $P$. However, there is only
one matrix in $\ker(Z)$ that is of rank two
\[ F=\begin{bmatrix} 0 & 1 & 0 \\ -1 & 0 & 0 \\ 0 & 0 & 0
\end{bmatrix}
\] and a unique reconstruction of $F$ is possible.
\end{exmp}

%%%%%%%%%%%%
\section {The 8-Point Algorithm and a Cube}
\label{8Failure}
%%%%%%%%%%%%
A quadric in $\PP^3$ is defined by the algebraic equation $p^TQp=0$, where $Q\in
\RR^{4\times 4}$ and $p \in \PP^3$. Since this is a quadratic equation in the indeterminate
$p$ we can choose the matrix $Q$ to be symmetric. Clearly $Q$ and any multiple of it
$\lambda Q,\, \lambda \in \RR$ define the same quadric. By a slight abuse of notation we
will refer to both the quadric and the matrix defining the quadric as $Q$.

One can try to fit a quadric
through a point configuration $P \in (\PP^3)^n$, then every point in the configuration
gives a linear equation on the ten entries of the  symmetric matrix $Q$. This results in
a linear equation system with indeterminate vector
\[\vv Q=[Q_{00},\ldots,Q_{04},Q_{11},\ldots,Q_{14},Q_{22},Q_{23},Q_{24},Q_{33},Q_{34},Q_{44}].\] Its
coefficient matrix can be constructed with the Veronese map.
The \emph{Veronese map} $\nu_{2,4}$ in degree two and four indeterminates is the map from
the four indeterminates to all monomials of degree two in these indeterminates.
\[
\begin{array}{rccc}
    \nu_{2,4}:&\PP^3&\rightarrow& \PP^9\\
&[x_1,x_2,x_3,x_4]&\mapsto&[x_1^2,
    x_1x_2,x_1x_3,x_1x_4,x_2^2,x_2x_3,x_2x_4,x_3^2,x_3x_4,x_4^2].\end{array}
\]

For convenience of notation we apply the map $\nu_{2,4}$ to each point in a configuration
separately. The map $\nu_{2,4}$ applied to $P\in(\PP^3)^n$ gives a matrix
$\nu_{2,4}(P)\in(\PP^9)^n$.  Therefore, if there exists a quadric $Q$ through the points in a
configuration $P\in(\PP^3)^n$, it can be computed via the linear equation system $\nu_{2,4}(P)\vv
Q=0$.

The rank of $Z$ is very essential to the 8-point algorithm and there is a relation to $\nu_{2,4}(P)$.
\begin{lem}\label{lem:ZandNu} 
Let $A_1$, $A_2$ be two cameras and $P\in(\PP^3)^n$ be a world point configuration. Let
$A_1P_i=\lambda_iX\in(\PP^2)^n$ and $A_2P_i=\lambda_iY\in(\PP^2)^n$, then the rank of $Z$ is bounded by the rank of
$\nu_{2,4}(P)$.
\end{lem}
\begin{proof} 
The statement can be rewritten. The number of linear independent equations in the system
of linear equations $y_i^TFx_i=0$ is bounded by the number of linear independent equations
in the system of linear equations $P_i^TQP_i=0$, where $Q$ is a $4\times 4$ generic
symmetric matrix of indeterminates.  Without loss of generality we assume that
$\nu_{2,4}(P)$ is of rank $k$ and the first $k$ equations $P_1^TQP_1,\ldots, P_k^TQP_k$ are
independent. Setting $Q=A_2^TFA_1$ we can write every point in the span of $Z$ as
$y^TFx=\sum_n\lambda_iy_i^TFx_i=\sum_n\lambda_iP_i^TA_2^TFA_1P_i=\sum_n\lambda_iP_i^TQP_i$
and by the independence of the first $k$ equations $P_1^TQP_1,\ldots, P_k^TQP_k$ we get
\[y^TFx=\sum_k
\lambda_iP_i^TQP_i=\sum_k\lambda_i y_i^TFx_i.\]
\end{proof}

\begin{figure}[b]
\center
\includegraphics[width=0.7\textwidth]{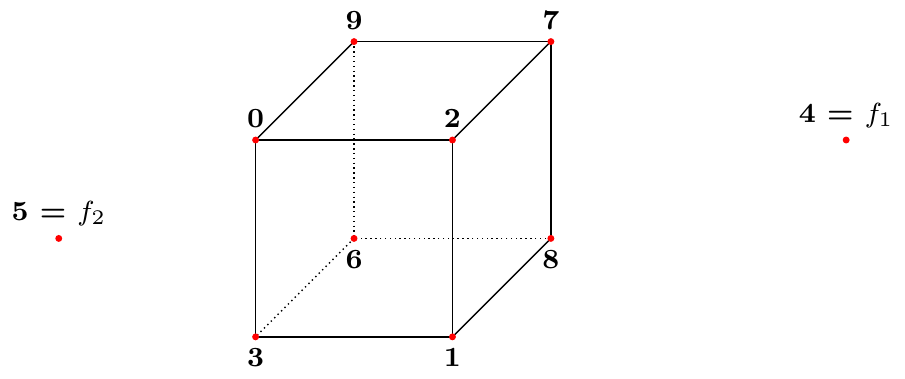}
\caption {The chosen labeling of word points and camera centers.}
\label{fig:cubeLabel}
\end{figure}

Let $P\in(\PP^3)^{10}$ represent a configuration of ten points.  The question whether the ten points
of $P$ are inscribable to a quaternary quadric has been studied in classical algebraic
geometry. However, no geometric interpretation is known up till now and it probably would
be too complicated to be of any use.  In algebraic terms this condition can easily be phrased
as $\det(\nu_{2,4}(P))=0$. Turnbull and Young give a description of the $PGL(3)$ invariant
$\det(\nu_{2,4}(P))=0$ in the bracket algebra (Turnbull-Young invariant)
\cite{turnbull1926linear}. Later the Turnbull-Young invariant has been straightened to a bracket
polynomial of degree 5 with 138 monomials \cite{white1988implementation}. Using the vertex labeling of Figure
\ref{fig:cubeLabel} the Turnbull-Young invariant given in \cite[p. 8-9]{white1988implementation}
reduces to the bracket polynomial
\begin{equation}
  \label{eq:fourBrackets}
  \begin{array}{l}
\left[\mathbf{0135}\right]\left[\mathbf{0247}\right]\left[\mathbf{1268}\right]\left[\mathbf{3469}\right]\left[\mathbf{5789}\right]-
\left[\mathbf{0134}\right]\left[\mathbf{0257}\right]\left[\mathbf{1268}\right]\left[\mathbf{3569}\right]\left[\mathbf{4789}\right]+\\
\left[\mathbf{0125}\right]\left[\mathbf{0346}\right]\left[\mathbf{1378}\right]\left[\mathbf{2479}\right]\left[\mathbf{5689}\right]-
\left[\mathbf{0124}\right]\left[\mathbf{0356}\right]\left[\mathbf{1378}\right]\left[\mathbf{2579}\right]\left[\mathbf{4689}\right],
 \end{array}
\end{equation}
where $\mathbf{0},\mathbf{1},\cdots, \mathbf{9}$ denote the points in the configuration
$[P,f_1,f_2]\in (\PP^3)^{10}$.

 We use this invariant to
show that $\nu_{2,4}(P)$ is not of full rank if the points in $P$ are the vertices of a combinatorial cube.

\begin{prop}\label{prop:cubeDegenerate}
Let $P$ be the vertices of a combinatorial cube in $\RR^3$, then the rank of $\nu_{2,4}(P)$ is at most seven.
\end{prop}
\begin{proof}
Consider the equation system $\nu_{2,4}([P,f_1,f_2])$. These are the equations
$P_i^TQP_i$ concatenated with the two equations $f_i^TQf_i$ of two arbitrary points $f_1,\,f_2\in
\PP^3$.  If the rank of $\nu_{2,4}(P)$ is at most seven, then the rank of equation system
$\nu_{2,4}([P,f_1,f_2]^T)$ is at most nine. Thus the ten points of the configuration $[P,f_1,f_2]\in
(\PP^3)^{10}$ are in special position and are inscribable to a quarternary quadric. This is
equivalent to satisfying the Turnbull-Young invariant. 
  We checked with \macaulay \cite{M2} that the polynomial of Equation \ref{eq:fourBrackets} vanishes
for all combinatorial cubes with the eight vertices
$P=[\mathbf{0},\mathbf{1},\mathbf{2},\mathbf{3},\mathbf{6},\mathbf{7},\mathbf{8}, \mathbf{9}]$
independent of the choice of the two points $f_1=[\mathbf 4]$ and $f_2=[\mathbf 5]$.\end{proof}
Without using some speed-ups and simplifications solving Equation \ref{eq:fourBrackets} is
computationally out of reach. Thus we performed the computation in \macaulay \cite{M2} as follows.
Since we are only interested in points in $\RR^3$ we fixed the fourth coordinate of every point to
one. Further we have the freedom of choice of a coordinate system in $\RR^3$. We choose the point
$\mathbf 0$ as the origin and the points $\mathbf {3}$, $\mathbf 2$, $\mathbf 9$ as the three unit
vectors. This implies that $\mathbf{1}$ is on the $xy$-plane, $\mathbf {5}$ is on the $xz$-plane and
$\mathbf {7}$ is on the $yz$-plane.
In  the Listing \ref{algoCritical} one can find the \macaulay \cite{M2} code  used to check
Proposition \ref{prop:cubeDegenerate}.
\begin{rem}
  Various statements about eight points and quadric surface are known in classical
  algebraic geometry, like the three-dimensional version of Miquel's Theorem
  \cite[p. 18]{bobenko2008discrete} and the statements about eight associated points \cite{turnbull1925vector}.
\end{rem}

Due to Proposition \ref{prop:cubeDegenerate} we understand the behavior of the 8-point
algorithm if its
input configurations $X,Y$ are images of the vertices of a combinatorial cube.

\begin{thm}\label{thm:8pointFails}
  Let $A_1,A_2$  be two arbitrary cameras and let $P\in (\PP^3)^8$ be the vertices of a combinatorial cube.  Then the $8$-point algorithm with input $A_1P,A_2P$ fails to compute a fundamental matrix $F$. It is defeated by the vertices of $P$.
\end{thm}
\begin{proof}
By Proposition \ref{prop:cubeDegenerate} the matrix $\nu_{2,4}(P)$ is of  rank seven at most,
thus by Lemma \ref{lem:ZandNu} the matrix $Z$ is of rank seven at most. Hence the assumption
in the $8$-point algorithm that $Z$ is of rank eight at least is not satisfied.  
\end{proof}

Theorem \ref{thm:8pointFails} states that if we take two pictures from the vertices of a
combinatorial cube, then the matrix $Z$ has at most rank seven and thus the 8-point algorithm is not
able to compute the fundamental matrix.
\begin{rem}
  Unlike to the conditions on critical configurations Theorem \ref{thm:8pointFails} does not impose
  any constraints on the camera centers.
\end{rem}

%%%%%%%%%%%%%%%%%%%
\section{Reconstruction of $F$ From Cubes}
%%%%%%%%%%%%%%%%%%%
Even if $\dim(\ker(Z))=2$ it is sometimes possible to reconstruct the fundamental matrix $F$.
Since the matrix $Z$ is of rank at most seven one can try to reconstruct the fundamental
matrix by additionally enforcing the singularity condition of the fundamental
matrix. Solving for $F$ means finding the real roots of a univariate cubic polynomial. However
for certain regions in $\RR^3\times \RR^3$ of the two focal points, this polynomial has
more than one solution and a unique reconstruction of $F$ is not possible. These regions are
semi-algebraic sets. We study these by first studying the simplest case, when $P$ are the
vertices of the \emph{unit cube} $C_u=1/2\cdot\conv(\pm e_1\pm e_2 \pm e_3)$.

For the three-dimensional unit cube $C_u$ the matrix $Q_u\in \text{Sym}_4(\RR)$ defining the
family of quadrics through its vertices diagonalizes to $Q_u=\diag(\alpha,\beta,\gamma,\delta)$,
$\alpha,\beta,\gamma,\delta\in \RR$, such that $\alpha+\beta+\gamma+\delta=0$. This results in a two
parameter family of quadrics running through the eight vertices of $C_u$. If we include the two
camera centers then there is exactly one quadric $Q$ running through all ten points and it is given
as the solution of the linear equation system
\begin{equation}\label{linearEquaQuadric}
\begin{array}{rcl}
\begin{array}{c}
\text{tr}(Q)=0\\
c_1^TQc_1=0\\
c_2^TQc_2=0
\end{array}&\Leftrightarrow&
\underbrace{\begin{bmatrix}
1&1&1&1\\
x_1^2&x_2^2&x_3^2&x_4^2\\
y_1^2&y_2^2&y_3^2&y_4^2\\
\end{bmatrix}}_{:=M}
\begin{bmatrix}
\alpha\\
\beta\\
\gamma\\
\delta
\end{bmatrix}=0
\end{array}
\end{equation}
where $f_1=[x_1, x_2, x_3, x_4]$ and $f_2=[y_1, y_2, y_3, y_4]$. 
\begin{figure}[t] \centering
\includegraphics[width=0.375\textwidth]{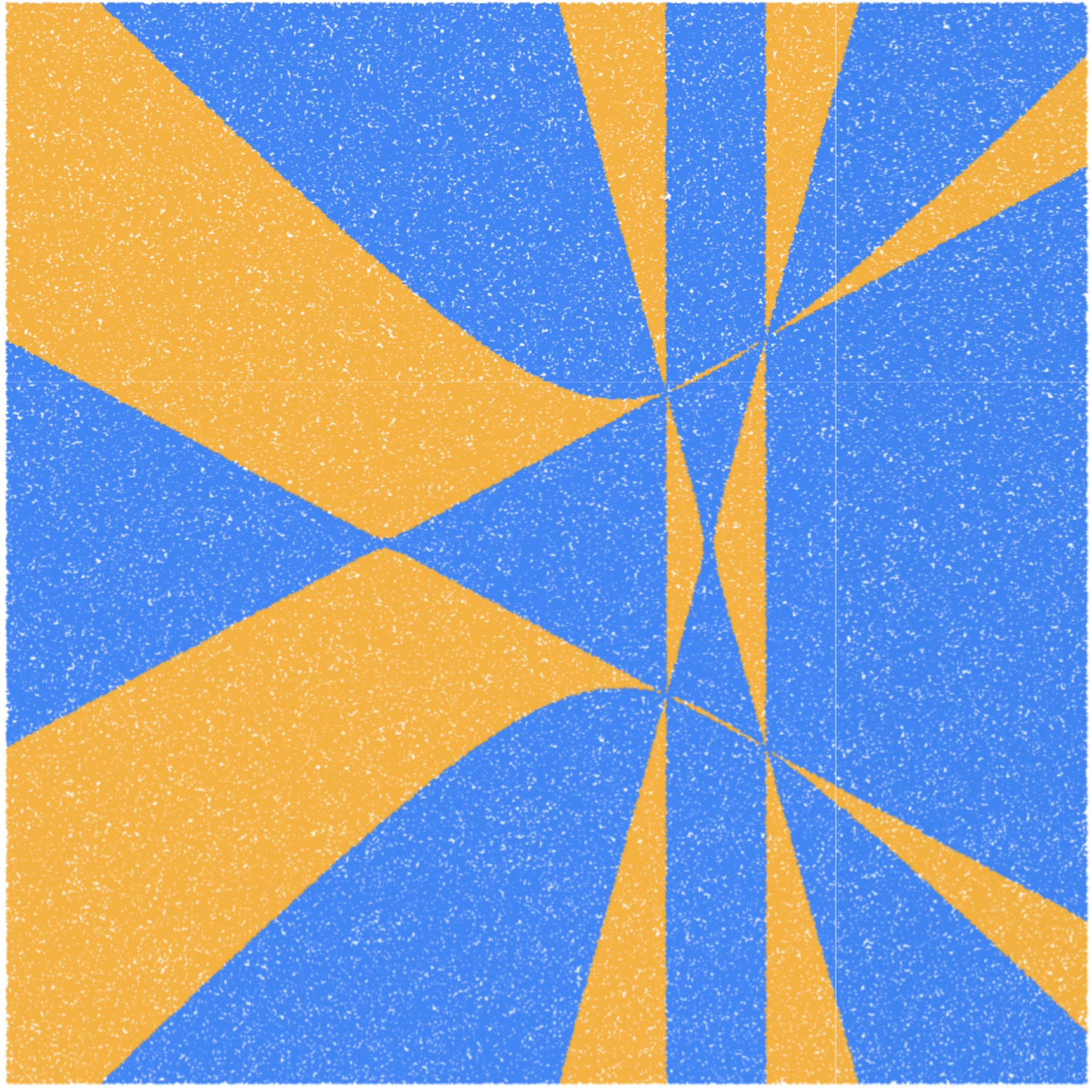}
\caption{Region of failure (orange) of the 8-point algorithm (enforcing the singularity constraint)
with a unit cube as input. We fixed the focal point of the first camera and randomly sampled the focal
point of the second camera in a chosen plane.}
\label{fig:semAlg}
\end{figure} By Cramer's rule
we construct the solution of this linear equation system, as the vector of the four signed maximal
minors of the matrix M,
\begin{equation}\label{eq:symetricSignature}
\begin{bmatrix}
\begin{vmatrix}
1&1&1\\
x_2^2&x_3^2&x_4^2\\
y_2^2&y_3^2&y_4^2\\
\end{vmatrix},&
-\begin{vmatrix}
1&1&1\\
x_1^2&x_3^2&x_4^2\\
y_1^2&y_3^2&y_4^2\\
\end{vmatrix},&
\begin{vmatrix}
1&1&1\\
x_1^2&x_2^2&x_4^2\\
y_1^2&y_2^2&y_4^2\\
\end{vmatrix},&
-\begin{vmatrix}
1&1&1\\
x_1^2&x_2^2&x_3^2\\
y_1^2&y_2^2&y_3^2\\
\end{vmatrix}\\
\end{bmatrix}.
\end{equation} 
We denote by $f^{*2}\in\PP^3$ the coordinate wise square of the vector $f\in\PP^3$.  Then Equation
\ref{eq:symetricSignature} can be written as the cross product of three vectors $(1,1,1,1)^T \times
f_1^{*2} \times f_2^{*2}$.

 Let $Q\in \text{Sym}_4(\RR)$ be the quadric running through the vertices of a combinatorial cube
$C$ and the two camera centers $f_1,f_2$ . Further let $Q_u \in \text{Sym}_4(\RR)$ be the quadric
running through the vertices of the unit cube $C_u$ and the cameras centers $g_1,g_2$.
\begin{prop}\label{prop:projCube} If there is a projective transformation $T\in PGL(3)$ from $C_u$ to
$C$, then $(1,1,1,1)^T \times f_1^{*2} \times f_2^{*2}$ has the same sign pattern as
$(1,1,1,1)\times Tg_1^{*2} \times Tg_2^{*2}$ and  $Q,Q_u$ have the same sign pattern.
\end{prop}
\begin{proof} Let $P\in (\PP^3)^{10}$ be the vertices of $C$ together with $f_1,f_2$ and $P_u$ be
the vertices of $C_u$ together with $Tg_1,Tg_2$.
\[(P_u^TQ_uP_u)_{ii}=0\,\forall\, i\Leftrightarrow(P^TT^TQ_uTP)_{ii}=0\,\forall\, i\Rightarrow
T^TQ_uT=Q.\] 
Now by Sylvester's law of inertia $Q$ and $Q_u$ have the same sign pattern.
\end{proof}

From Proposition \ref{prop:projCube} we are able to compute the type of the quadric $Q$ by
finding a projective transformation that maps the vertices of $C$ to the vertices of
$C_u$. In particular $Q$ is ruled if $T^TQ_uT=Q$ is ruled.

\begin{rem}
If we interprete the vector of diagonal entries of $Q$ as a point in $\PP^3$ the condition
$(1,1,1,1)^T \times f_1^{*2} \times f_2^{*2}$ is equivalent to $[\alpha,\beta,\gamma,\delta]$ being
on the intersection of the three planes with normal vectors $[1,1,1,1], \, f_1^{*2}, \, f_2^{*2}\in\PP^3$. 
\end{rem}

The quadric is ruled if $[\alpha,\beta,\gamma,\delta]$ has a sign pattern of the following types
$[-,-,+,+]$, $[+,-,+,-]$, or $[-,+,+,-]$.  The boundaries of the components of the semi-algebraic
set (where the signature of the quadric changes) are given as the vanishing set of the determinants
$\alpha$ , $\beta$, $\gamma$ and $\delta$ of Equation \ref{eq:symetricSignature}.

In some cases the signature changes, but
still the quadric stays non-ruled. Therefore, to get a more explicit answer for this example it is useful
to break up the symmetry of Equation \ref{eq:symetricSignature}. Since the quadric is independent on
a scaling factor of $Q$ we can set the last diagonal entry without loss of
generality  to $\delta=1$. Thus the equation system of Equation
\ref{linearEquaQuadric} degenerates to an equation system of three equations in three variables and
we can solve it explicitly.
Then $Q=\diag(\alpha,\beta,-\alpha-\beta-1,1)$.
There are two distinct cases when the quadric is ruled: 
\begin{enumerate}
\item If $\alpha,\beta\leq 0$ and $\alpha+\beta\leq -1$. 
\item If $\alpha, \, \beta$ have different signs and $\alpha+\beta\geq 1$.
\end{enumerate}

The vector of diagonal
entries of $Q$ then is given up to scale by
\[\footnotesize
\begin{bmatrix}
\begin{vmatrix} x_2^2-x_3^2 & x_3^2-1\\ y_2^2-y_3^2 & y_3^2-1
\end{vmatrix}\\[12pt]
\begin{vmatrix} x_1^2-x_2^2 & x_3^2-1\\ y_1^2-y_2^2 & y_3^2-1
\end{vmatrix}\\[12pt] -\begin{vmatrix} x_2^2-x_3^2 & x_3^2-1\\ y_2^2-y_3^2 & y_3^2-1
\end{vmatrix}\!-\!\begin{vmatrix} x_1^2-x_2^2 & x_3^2-1\\ y_1^2-y_2^2 & y_3^2-1
\end{vmatrix}-\begin{vmatrix} x_1^2-x_2^2 & x_2^2-x_3^2\\ y_1^2-y_2^2 & y_2^2-y_3^2
\end{vmatrix}\\[12pt]
\,\begin{vmatrix} x_1^2-x_2^2 & x_2^2-x_3^2\\ y_1^2-y_2^2 & y_2^2-y_3^2
\end{vmatrix}
\end{bmatrix}\in \PP^3
\]

%%%%%%%%%%%%
\section{How to Handle Pictures of Cubes}
%%%%%%%%%%%%

In the sections above we did not consider noisy pictures. As seen in Theorem
\ref{thm:8pointFails} the rank of $Z$ drops by at least one if we take pictures of
combinatorial cubes. However, in the presence of noise the matrix $Z\in \RR^{8\times 9}$
again is of rank eight, but it is close to being singular. This results in a very bad
performance of the algorithm in practice and we strongly advise against using it. It is
simply the wrong choice of algorithm, since it is incapable of dealing with this set-up.

As discussed in Section \ref{8Failure} the matrix $Z$ has at most rank seven in the noisefree case, hence a standard implementation of the $7$-point algorithm run on a 7-element
subset of vertices can be used to retrieve the fundamental matrix. Thus a natural fix for
the flaw of the $8$-point algorithm is to use a modified version of the $7$-point
algorithm, that allows eight points as inputs. We use singular value decomposition on $Z$
to obtain a matrix $Z'\in \RR^{8\times 9}$ that is of rank seven and minimizes the Frobenius
norm of $|Z-Z'|$. Let $Z=UDV^T$ be the singular value decomposition of $Z$, then by the
Eckart-Young-Mirsky theorem $Z'=U\diag(\sigma_1,\ldots\sigma_7,0,0)V^T$. Now we use the
$7$-point algorithm to obtain (one to three) possible solutions for the fundamental
matrix. If there are multiple solutions, we chose the one that minimizes the residual
error on the input points.  The $8$-point algorithm for cubes then is given in Algorithm \ref{algo:cube-8-point}.
\begin{figure}
\begin{algorithm}[H]\label{algo:cube-8-point}
 \KwIn{Two image point configurations $X,Y\in(\PP^2)^n$;}
 \KwOut{The fundamental matrix $F$;}
\Begin{
\begin{enumerate}
\item Normalize $X,Y$
\item Compute $Z\in \RR^{8\times 9} $ from $X,Y$;
\item Compute $Z'$ that minimizes the Frobenius norm of $|Z-Z'|$.\\
$Z':=U\diag(\sigma_1,\ldots\sigma_7,0,0)V^T$.
\item Compute the two generators $f_1$ and $f_2$ of
$\ker(Z')$\\ and solve $\det(\alpha F_1+(1-\alpha)F_2)=0$
\end{enumerate}
\If{$\det(\alpha F_1+(1-\alpha)F_2)$\emph{ has multiple real roots} }{
Choose the solution that minimizes the residual
error on $X,Y$.}
}
 \caption{Cube-8-point algorithm}
\end{algorithm}
\end{figure}

%%%%%%%%%%%%
\section{Numerical Experiments}
%%%%%%%%%%%%

We performed random tests on synthetic data to compare the performance of different 
algorithms. To do so, we sampled random cubes within the box $[-1,1]^3$. The cameras were
chosen with focal points roughly on a sphere with radius six. Gaussian noise was applied
onto the images with standard deviation between $0\%-10\%$ of the image sizes and zero
mean. For each noise level we chose 2000 random samples and respectively computed 2000 approximations
of fundamental matrices.  As a measure to analyze the results of the
algorithm we used the metric on the Grassmanian between two linear subspaces, namely the
angle between the vectorization of the true fundamental matrix and the approximated
fundamental matrix.

\begin{figure}[b]
  \includegraphics[width=\textwidth]{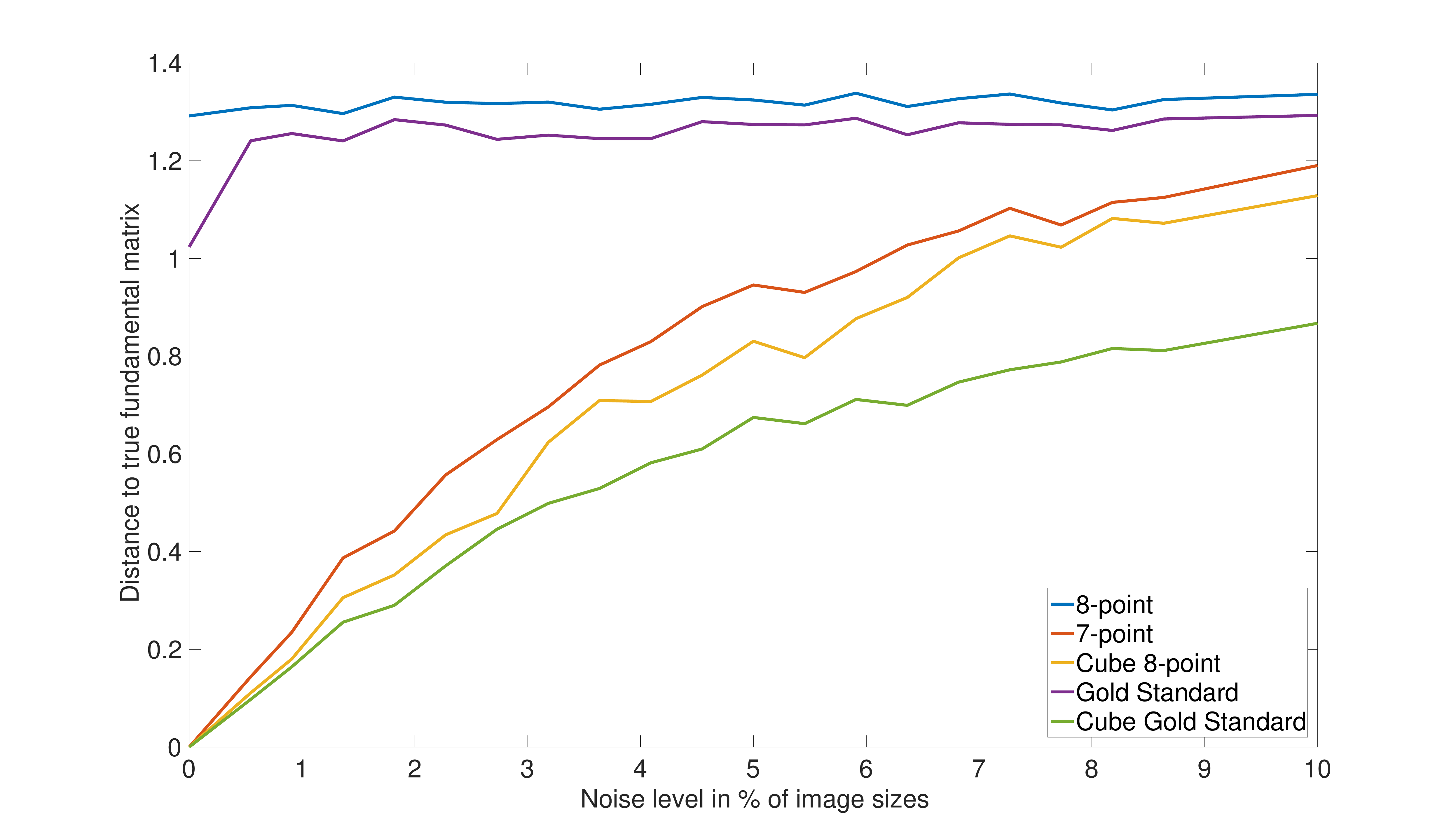}%
   \caption{Comparison between algorithms to reconstruct the fundamental matrix.  Plotted is the distance as angle (rad) between two one-dimensional subspaces versus the noise level.}
  \label{fig:comparisonFunda}
\end{figure} 
The modified version of the $8$-point algorithm for cubes (Algorithm
\ref{algo:cube-8-point}) gives good results. Its running time is almost the same as of the usual
$7$-point algorithm, but unlike the unmodified version it is noise correcting. It also
gives better results than algorithms that rely on non-linear optimization to estimate the
fundamental matrix, like the \emph{Gold Standard algorithm} \cite[Algorithm
11.3]{hartley2003multiple}. These algorithms usually use an initial guess of the
fundamental matrix computed via the $8$-point algorithm. But the results of the $8$-point
algorithm are so far off from the true fundamental matrix, such that the non-linear
solvers get stuck in a local optimum far off from the global one.  However by using estimates of the fundamental matrix
computed with the modified $8$-point algorithm for cubes (Algorithm
\ref{algo:cube-8-point}) the Gold Standard algorithm can be improved. For example using an
estimate of the fundamental matrix computed with Algorithm \ref{algo:cube-8-point} as
initial input, instead of a fundamental matrix computed with the 8-point algorithm
\cite[Algorithm 11.1]{hartley2003multiple}, improves the Gold Standard algorithm
\cite[Algorithm 11.3]{hartley2003multiple}. In Figure \ref{fig:comparisonFunda} this
version of the Gold Standard algorithm is denoted by \emph{Cube Gold Standard}.

Note that there are also global solvers to find fundamental matrices based on semidefinite
programming \cite{bugarin2015rank}. Noteworthy the situation in our case is different from
the one depicted \cite{bugarin2015rank}. Usually there are more than 8 correspondent
image point pairs available. In \cite{bugarin2015rank}  ten points and more are considered. Thus due to noisy data the matrix $Z$ is of rank nine. For the pathological case of only eight points and a rank drop in $Z$, Algorithm 1 in
\cite{bugarin2015rank} has not been able to certify global optimality
 based on \texttt{GloptiPoly 3} \cite{henrion2009gloptipoly}. 

If one has the freedom of choice to place the eight points in $\PP^3$ we suggest using the skew
octagon for a more robust reconstruction of the fundamental matrix. The skew octagon is
computational the optimal solution to various sphere placement problems, e.g. the Thomson problem \cite{skewOctagon}.

\subsection{Acknowledgment}
We would like to thank Michael Joswig for his guidance and Fredrik Kahl for our correspondences about critical configurations.
\section{Computations}
\label{compu3}
Below you can find the code we used to check Proposition \ref{prop:cubeDegenerate}.
Lines 2-11 define the vertices of a cube.
Lines 14-17 define the facets of the cube. In lines 22-34 the reduced Turnbull-Young invariant of Equation \ref{eq:fourBrackets} is defined.
This invariant vanishes in the quotient ring $S=R/J$ of line 19.

\begin{lstlisting}  [label=algoCritical, caption=Vanishing of Turnbull-Young Invariant]
R=QQ[x_1..x_15]
Cube=matrix{ {0,0,0,1},           --0    
             {x_1,x_2,0,1},    	  --1    
             {0,1,0,1},	      	  --2    
             {1,0,0,1},	      	  --3    
             {x_10,x_11,x_12,1},  --4 f_1
             {x_13,x_14,x_15,1},  --5 f_2
             {x_5,0,x_6,1},    	  --6    
             {0,x_3,x_4,1},    	  --7    
             {x_7,x_8,x_9,1},	  --8    
             {0,0,1,1}}	      	  --9    
-- the six facets of the cube are coplanar
-- their determinants vanish
J=ideal(
det submatrix(Cube,{0,1,2,3},),det submatrix(Cube,{6,7,8,9},),
det submatrix(Cube,{0,3,6,9},),det submatrix(Cube,{1,2,7,8},),
det submatrix(Cube,{0,2,7,9},),det submatrix(Cube,{1,3,6,8},))

S=R/J
F=map(S,R)
-- the remaining 4 bracket monomials of the Turnbull-Young invariant 
p=F(
det submatrix(Cube,{0,1,3,5},)*det submatrix(Cube,{0,2,4,7},)*
det submatrix(Cube,{1,2,6,8},)*det submatrix(Cube,{3,4,6,9},)*
det submatrix(Cube,{5,7,8,9},)-
det submatrix(Cube,{0,1,3,4},)*det submatrix(Cube,{0,2,5,7},)*
det submatrix(Cube,{1,2,6,8},)*det submatrix(Cube,{3,5,6,9},)*
det submatrix(Cube,{4,7,8,9},)+
det submatrix(Cube,{0,1,2,5},)*det submatrix(Cube,{0,3,4,6},)*
det submatrix(Cube,{1,3,7,8},)*det submatrix(Cube,{2,4,7,9},)*
det submatrix(Cube,{5,6,8,9},)-
det submatrix(Cube,{0,1,2,4},)*det submatrix(Cube,{0,3,5,6},)*
det submatrix(Cube,{1,3,7,8},)*det submatrix(Cube,{2,5,7,9},)*
det submatrix(Cube,{4,6,8,9},))
\end{lstlisting}

 \bibliographystyle{amsplain} 
\bibliography{example}
\end{document}